\documentclass[11pt]{article}

\usepackage[preprint]{acl}

\usepackage{times}
\usepackage{latexsym}
\usepackage[T1]{fontenc}
\usepackage[utf8]{inputenc}\usepackage{microtype}
\usepackage{inconsolata}
\usepackage{graphicx}
\usepackage{amsmath,amssymb}
\usepackage{amsthm}
\newtheorem{proposition}{Proposition}
\usepackage{booktabs}\usepackage{tabularx}\usepackage{colortbl}
\usepackage{multirow}
\usepackage{pgfplots}
\pgfplotsset{compat=1.18}
\usetikzlibrary{svg.path}
\newcommand{\openailogo}{\scalebox{0.0195}{\begin{tikzpicture}[yscale=-1]\fill[even odd rule] svg "M196.4 185.8l0-48.6c0-4.1 1.5-7.2 5.1-9.2l97.8-56.3c13.3-7.7 29.2-11.3 45.6-11.3 61.4 0 100.4 47.6 100.4 98.3 0 3.6 0 7.7-.5 11.8L343.3 111.1c-6.1-3.6-12.3-3.6-18.4 0L196.4 185.8zM424.7 375.2l0-116.2c0-7.2-3.1-12.3-9.2-15.9L287 168.4 329 144.3c3.6-2 6.7-2 10.2 0L437 200.7c28.2 16.4 47.1 51.2 47.1 85 0 38.9-23 74.8-59.4 89.6l0 0zM166.2 272.8l-42-24.6c-3.6-2-5.1-5.1-5.1-9.2l0-112.6c0-54.8 42-96.3 98.8-96.3 21.5 0 41.5 7.2 58.4 20L175.4 108.5c-6.1 3.6-9.2 8.7-9.2 15.9l0 148.5 0 0zm90.4 52.2l-60.2-33.8 0-71.7 60.2-33.8 60.2 33.8 0 71.7-60.2 33.8zm38.7 155.7c-21.5 0-41.5-7.2-58.4-20l100.9-58.4c6.1-3.6 9.2-8.7 9.2-15.9l0-148.5 42.5 24.6c3.6 2 5.1 5.1 5.1 9.2l0 112.6c0 54.8-42.5 96.3-99.3 96.3l0 0zM173.8 366.5L76.1 310.2c-28.2-16.4-47.1-51.2-47.1-85 0-39.4 23.6-74.8 59.9-89.6l0 116.7c0 7.2 3.1 12.3 9.2 15.9l128 74.2-42 24.1c-3.6 2-6.7 2-10.2 0zm-5.6 84c-57.9 0-100.4-43.5-100.4-97.3 0-4.1 .5-8.2 1-12.3l100.9 58.4c6.1 3.6 12.3 3.6 18.4 0l128.5-74.2 0 48.6c0 4.1-1.5 7.2-5.1 9.2l-97.8 56.3c-13.3 7.7-29.2 11.3-45.6 11.3l0 0zm127 60.9c62 0 113.7-44 125.4-102.4 57.3-14.9 94.2-68.6 94.2-123.4 0-35.8-15.4-70.7-43-95.7 2.6-10.8 4.1-21.5 4.1-32.3 0-73.2-59.4-128-128-128-13.8 0-27.1 2-40.4 6.7-23-22.5-54.8-36.9-89.6-36.9-62 0-113.7 44-125.4 102.4-57.3 14.8-94.2 68.6-94.2 123.4 0 35.8 15.4 70.7 43 95.7-2.6 10.8-4.1 21.5-4.1 32.3 0 73.2 59.4 128 128 128 13.8 0 27.1-2 40.4-6.7 23 22.5 54.8 36.9 89.6 36.9z";\end{tikzpicture}}}
\usepackage{listings}
\lstdefinestyle{prompt}{
  basicstyle=\ttfamily\scriptsize,
  breaklines=true,
  breakatwhitespace=true,
  breakautoindent=false,
  breakindent=0pt,
  columns=fullflexible,
  keepspaces=true,
  showstringspaces=false,
  frame=single,
  framerule=0.3pt,
  xleftmargin=0.5em,
  xrightmargin=0.5em,
  aboveskip=0.5\baselineskip,
  belowskip=0.5\baselineskip
}

\definecolor{DecisionTop}{HTML}{3569A8}
\definecolor{DecisionOther}{HTML}{9FB4CC}
\definecolor{DecisionOutside}{HTML}{D9D9D9}
\definecolor{OursColor}{HTML}{3569A8}
\definecolor{ReActColor}{HTML}{D88743}

\title{Self-Improving Agents with World Model Predictive Paths}

\author{
  Haori Lu\thanks{\ Corresponding author.} \quad Tianyu Gao \quad Yaoyao Liu \quad Yingying Li \\
  University of Illinois Urbana-Champaign \\
  \texttt{\{haorilu2, tianyu30, lyy, yl101\}@illinois.edu}
}

\begin{document}
\maketitle

\begin{abstract}
Large language model (LLM) agents are increasingly expected to interact with environments and improve their decisions during deployment.
However, how anticipating future outcomes and drawing on accumulated interaction experience can jointly support decision-making remains underexplored.
We propose a test-time framework, inspired by Model Predictive Path Integral (MPPI) control, that combines multi-branch predictive analysis with experience-based adaptation.
At each step, the agent proposes multiple candidate actions and imagines their outcomes with a world model, so that alternatives can be compared without being executed in the environment.
Since language environments provide neither explicit dynamics nor costs, the predicted outcomes are evaluated for task progress and risk, and MPPI-style weights over the resulting costs guide the final decision.
A dual-branch experience memory enables self-improvement: the planner branch refines candidate generation and action selection, while the world model branch improves next-observation prediction.
Experiments on ALFWorld and ScienceWorld across three backbones demonstrate the strong overall performance of our method and the effectiveness of experience-guided predictive control. Our code will be made publicly available.
\end{abstract}

\section{Introduction}

Large language model (LLM) agents interact with environments by interpreting observations and taking actions to accomplish tasks.
Long-horizon tasks require them to sustain progress under incomplete observations, where poor decisions can lead to repeated failures or wasted interactions.
Building agents that make reliable decisions over extended interactions therefore remains a central challenge.

Prior work improves interactive decision making in several complementary ways.
ReAct interleaves reasoning with environment feedback~\citep{yao2023react}, but commits to a single proposed action at each step, which may miss useful alternatives.
Search-based methods instead explore multiple reasoning or action paths at inference time~\citep{yao2023tot,zhou2023lats,hao2023rap,liu2026imagine}, but their cost grows with the breadth and depth of the search.
A separate line improves agents through retrieved reflections and trajectory-derived action advantages~\citep{shinn2023reflexion,jitrl2026}.
These directions motivate exploring a few alternatives before acting, without executing them, together with experience gathered through interaction.
This also resembles how people mentally simulate a few alternatives before choosing~\citep{gilbert2007prospection,vul2014one}: the agent imagines the outcomes of several candidate actions~\citep{hafner2020dreamer,gu2025webdreamer}, which avoids the interaction steps and potential irreversible consequences of trying them in the environment.
Model Predictive Path Integral (MPPI) control offers a concrete recipe for this pattern: it samples a finite batch of candidate trajectories, weights them by predicted costs with exponential (Gibbs) weights, executes only the next action, and replans~\citep{williams2015mppi,williams2017information}.

Language environments, however, provide neither explicit dynamics nor costs.
A language world model can supply the missing dynamics by predicting the outcome of each candidate action~\citep{hao2023rap,liu2026imagine,zhang2026worldevolver}, but better next-observation prediction does not necessarily yield higher task success~\citep{zhang2026worldevolver}.
We therefore couple prediction directly with decision making: the agent predicts the outcomes of several candidate actions, an evaluator assigns costs for task progress and risk, and MPPI-style weights summarize these costs for a final controller that selects or synthesizes the action to execute.

To make both the candidates and their predictions more reliable over time, we further introduce a dual-branch experience memory for self-improvement.
The planner branch uses accumulated experience to improve candidate generation and action selection, while the world model branch improves next-observation prediction, so that a single stream of interaction experience improves both what is proposed and what is predicted.
Both branches operate through retrieved textual context, with policy and world model parameters frozen.
Our analysis bounds the agent’s per-step regret in terms of three error sources, namely limited candidate coverage, errors in action selection, and inaccurate next-observation predictions, which motivates the two branches.

Experiments on ALFWorld~\citep{shridhar2021alfworld} and ScienceWorld~\citep{wang2022scienceworld} demonstrate the strongest overall success-rate performance among the evaluated methods across three backbones.
Our main contributions are:
\begin{itemize}
    \item We propose a test-time, MPPI-inspired framework that improves a bounded predict--select--replan decision loop with experience: a world model predicts the outcomes of a few candidate actions, and a dual-branch memory improves both candidate generation and selection and next-observation prediction.
    \item We provide a theoretical analysis showing that the weighting rule is optimal under a KL-regularized objective, and that the per-step regret is bounded in terms of candidate coverage, action selection, and next-observation prediction errors, which motivates the two memory branches.
    \item Experiments on ALFWorld and ScienceWorld across three backbones demonstrate the strongest overall success rate performance among baselines. We also provide ablation studies and a computational cost analysis.
\end{itemize}

\section{Related Work}

\paragraph{LLM agents.}
LLM agents choose actions from textual or visual observations~\citep{wu2025number,wu2025kris} in a loop.
ReAct~\citep{yao2023react} and SayCan~\citep{ahn2022saycan} ground actions in environmental feedback or available skills, ADaPT~\citep{prasad2024adapt} decomposes tasks, and further methods support lookahead through branching or retrieved information~\citep{yao2023tot,zhou2023lats,koh2024tree,holt2026lwm}.
In traditional control, a related perspective also uses lookahead: model-predictive control~\citep{camacho2007mpc} plans over a finite horizon, executes only the first action, and replans from the new observation.
MPPI further weights sampled trajectories by their predicted costs~\citep{williams2015mppi,williams2017information}.
We adapt MPPI’s strategy of predicting and weighting multiple candidates to language agents through world model prediction, replacing its continuous control update with a discrete action update driven by language models, and thus call our framework MPPI-inspired.

\paragraph{World models for agent planning.}
World models can estimate how actions change an environment before they are executed.
Dynalang~\citep{lin2024dynalang}, RAP~\citep{hao2023rap}, WebDreamer~\citep{gu2025webdreamer}, and IMPLEMENT~\citep{song2026implement} simulate action outcomes with language-based or lightweight world models; while ITP~\citep{liu2026imagine}, WKM~\citep{qiao2024wkm}, and WorldEvolver~\citep{zhang2026worldevolver} add adaptive lookahead, structured knowledge, or retrieved experience.
Qian et al.~\citep{qian2026current} and Li et al.~\citep{li2026word} study how well agents use forecasts and whether LLMs can serve as text-based world models.
More recently, MemWM~\citep{wang2026memwm} trains a world model with offline-built memory to evaluate the candidates.
Unlike them, we keep all models frozen and use experience gathered during interaction in two branches that improve next-observation prediction and candidate generation and selection.

\paragraph{Self-improving agents.}
Experience-based methods improve later decisions by retaining information from prior interactions, building on memory systems~\citep{park2023generative,packer2023memgpt,chhikara2025mem0,xu2025amem,sumers2024coala,hu2025memory, gao2025selfevolving,liu2023online,liu2024wakening,lu2026rose}.
Reflexion~\citep{shinn2023reflexion}, ExpeL~\citep{zhao2023expel}, ERL~\citep{allard2026erl}, and ReasoningBank~\citep{ouyang2026reasoningbank} distill feedback or trajectories into experience, Voyager~\citep{wang2023voyager} and AWM~\citep{wang2024awm} retain skills or task procedures, AdaPlanner~\citep{sun2023adaplanner} and JitRL~\citep{jitrl2026} revise plans or action preferences from retrieved experience, and Memento~\citep{zhou2025memento}, ACE~\citep{zhang2025ace}, and Dynamic Cheatsheet~\citep{suzgun2026dynamic} adapt through episodic memory or evolving contexts instead of fine-tuning~\citep{wu2026generalization}.
These methods improve what the agent proposes; our dual-branch memory also improves what it predicts, so a single stream of interaction experience sharpens both the candidates and the forecasts used to choose among them.

\begin{figure*}[t]
\centering
\includegraphics[width=0.99\textwidth]{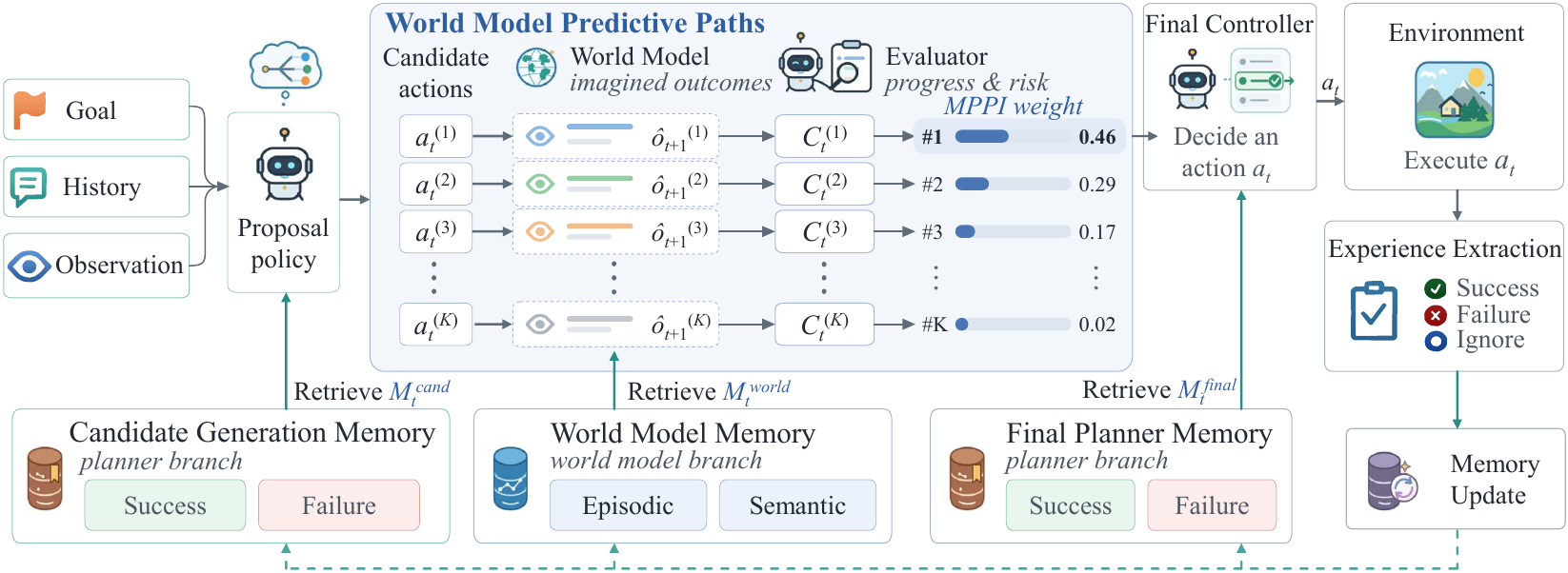}
\caption{Overview of our method. At step $t$, the proposal policy takes the goal, history, and observation and proposes $K$ candidate actions $a_t^{(k)}$. The world model predicts the next observation $\hat o_{t+1}^{(k)}$ of each candidate, the evaluator scores task progress and risk to give costs $C_t^{(k)}$, and MPPI-style weights summarize these costs for the final controller, which decides the executed action $a_t$. After execution, the transition updates both experience branches. The planner branch labels it as success, failure, or ignore and summarizes the first two into $M_t^{\mathrm{cand}}$ and $M_t^{\mathrm{final}}$, retrieved by the proposal policy and final controller in future steps; the world model branch updates its episodic and semantic memory $M_t^{\mathrm{world}}$, retrieved by the world model in future steps.}
\label{fig:loop}
\end{figure*}

\section{Problem Formulation}

We consider a language agent solving a long-horizon interactive task with goal $g$ and an interaction budget of $T$ steps.
At step $t$, the agent receives an observation $o_t$ and has interaction history $h_t=(o_1,a_1,\ldots,o_{t-1},a_{t-1})$.
It acts from context $x_t=(g,h_t,o_t)$.   Its policy $\pi(\cdot\mid x_t)$ defines a distribution over language actions, and the action at step $t$ is generated according to:
\begin{equation}\label{equ: def of pi}
a_t\sim\pi(\cdot\mid x_t),\qquad a_t\in\mathcal{A}.
\end{equation}

After $a_t$ is applied, the environment generates the next observation according to:
\begin{equation}
o_{t+1}\sim P_{\mathrm{env}}(o'\mid h_t,o_t,a_t).
\end{equation}

Our objective is to maximize the probability of completing the task within the interaction budget:
\begin{equation}
\pi^\star=\arg\max_{\pi}\Pr(\text{success by }T\mid g,\pi).
\end{equation}

In control, Model-Predictive Control (MPC) plans over a finite horizon $H$ by selecting the candidate action sequence $\mathbf{a}=(a_t,\ldots,a_{t+H-1})$ with the lowest expected predicted cost:
\begin{equation}
\mathbf{a}^{\star}=\arg\min_{\mathbf{a}}\bar{J}_t(\mathbf{a}),
\end{equation}
where $\bar{J}_t$ is the expected cost of the model-predicted trajectory~\citep{camacho2007mpc,rawlings2017model,li2019online,li2020online,zhang2021regret,li2021online}.
MPC executes only the first action and replans after the next observation.
Model Predictive Path Integral (MPPI)~\citep{williams2015mppi,williams2017information} approximates this search by sampling candidate trajectories and favoring those with lower predicted costs with Gibbs weights.
This predict-and-replan cycle suits our agent tasks, where each action shapes later decisions.
Inspired by MPPI, we develop a test-time framework for language agents that uses a world model to evaluate multiple candidate actions and guide decision-making.
The next section presents our method in detail.

\section{Method}

We first describe the test-time loop: a world model predicts outcomes for up to $K$ proposed actions, and MPPI-inspired weighting of task-conditioned costs guides a final controller to select or synthesize an action before replanning from the new observation.
We then describe how dual-branch experience memory improves planning and grounds future forecasts.

\subsection{MPPI-Inspired Decision Loop}

At time step $t$, the agent receives the goal, interaction history, current observation, and action protocol.
Let $M_t^{\mathrm{P}}$ denote the planner experiences retrieved for this decision, which will be introduced in Sec.~\ref{sec:method-experience} in detail.
The proposal module first reasons over the context, and then emits up to $K$ atomic candidate actions.
The resulting candidate set is represented as:
\begin{equation}\label{eq:proposal-sampling}
\begin{aligned}
\mathcal{A}_t&=\{a_t^{(k)}\}_{k=1}^{N_t}, & 1\leq N_t\leq K,\\
a_t^{(k)}&\sim q_\pi(a\mid x_t,M_t^{\mathrm{P}}), & k=1,\ldots,N_t.
\end{aligned}
\end{equation}

Here $q_\pi$ denotes the proposal policy, $K$ is the maximum number of candidates, and $N_t$ is the number generated at step $t$.
This adapts MPPI's candidate sampling to language actions: instead of adding perturbations to a continuous control sequence, the proposal policy generates candidate actions that could advance the task, gather information, or recover from failures.

Next, for each proposal, the world model receives the goal, current observation, candidate action, and retrieved world model memory $M_t^{\mathrm{world}}$.

It predicts a one-step textual observation of the action's immediate consequence:
\begin{equation}
\hat{o}_{t+1}^{(k)}
\sim
\hat{P}(o^{\prime}\mid g,o_t,a_t^{(k)},M_t^{\mathrm{world}}).
\end{equation}
We use one-step prediction rather than multi-step rollouts to simplify the architecture and control computational cost.
We denote the resulting candidate-specific imagined transition by
\begin{equation}
\hat{\tau}_t^{(k)}
=
(g,h_t,o_t,a_t^{(k)},\hat{o}_{t+1}^{(k)}),
\end{equation}
which bundles the task context, proposed action, and predicted outcome used to evaluate that candidate.
In classical MPPI, candidate trajectories are evaluated by a numerical cost function.
For language agents, we instantiate this cost-evaluation stage through textual outcome prediction and evaluation of task progress and risk.
A language evaluator maps each imagined transition to a bounded predictive cost:
\begin{equation}
C_t^{(k)}
=
C_\psi(g,h_t,o_t,a_t^{(k)},\hat{o}_{t+1}^{(k)})
\in[0,100],
\end{equation}
where lower cost indicates greater predicted usefulness for the goal.
The evaluator considers expected task progress, information value, and whether the predicted outcomes are consistent with the available evidence.
To account for uncertainty in predicted costs, we optimize a distribution $\mathbf w$ that favors low-cost candidates while retaining alternatives.
We denote its regularized objective by $\mathcal F_t(\mathbf w)$ and its minimizing distribution by $\mathbf w_t^\star$; the objective combines weighted predictive cost with a KL penalty from uniform weighting.

\begin{proposition}[Optimal candidate weighting]\label{prop:weighting}
For $N_t\geq 1$ candidates with finite costs, let $\Delta_{N_t}$ denote the probability simplex and $u_t^{(k)}=1/N_t$ the uniform reference. For $\lambda>0$, consider
\begin{equation}
\begin{aligned}
\mathbf w_t^\star
&=\arg\min_{\mathbf w\in\Delta_{N_t}}\mathcal F_t(\mathbf w),\\
\mathcal F_t(\mathbf w)
&=\sum_{k=1}^{N_t}w^{(k)}C_t^{(k)}
+\lambda D_{\mathrm{KL}}(\mathbf w\Vert\mathbf u_t),
\end{aligned}
\label{eq:candidate-objective}
\end{equation}
The objective has a unique minimizer
\begin{equation}
w_t^{(k)\star}
=\frac{\exp(-C_t^{(k)}/\lambda)}
{\sum_{j=1}^{N_t}\exp(-C_t^{(j)}/\lambda)}.
\label{eq:candidate-weights}
\end{equation}
For every $\mathbf w\in\Delta_{N_t}$,
\begin{equation}
\mathcal F_t(\mathbf w)-\mathcal F_t(\mathbf w_t^\star)
=\lambda D_{\mathrm{KL}}(\mathbf w\Vert\mathbf w_t^\star)\geq 0.
\end{equation}
\end{proposition}
The minimizing distribution is the standard Gibbs cost-tilted form underlying MPPI~\citep{williams2017information}, which also admits a free-energy interpretation of bounded-rational choice~\citep{ortega2013thermodynamics}.
The proof is in Appendix~\ref{app:weighting-proofs}.
Proposition~\ref{prop:weighting} shows that candidate weights depend on cost differences scaled by $\lambda$: smaller $\lambda$ emphasizes the lowest predicted costs, whereas larger $\lambda$ retains more balanced influence across candidates.
Our main configuration uses $\lambda=10$.
We use $w_t^{(k)}=w_t^{(k)\star}$ below.
The weights then guide the final controller, which also receives the goal, trajectory, observation, candidate set, predicted outcomes, and retrieved planner experiences.
We denote its action-generation step by
\begin{equation}\label{eq:final-controller}
\begin{aligned}
\mathcal{E}_t&=\{(\hat{o}_{t+1}^{(k)},w_t^{(k)})\}_{k=1}^{N_t},\\
a_t&\sim\pi_{\mathrm{final}}(\cdot\mid x_t,\mathcal{A}_t,\mathcal{E}_t,M_t^{\mathrm{P}}),
\end{aligned}
\end{equation}

where $\pi_{\mathrm{final}}$ is the final controller.
It generates an action, selecting a proposed action or synthesizing a new atomic action when useful.
The controller executes only the selected action in the real environment.
It then appends the resulting observation to the trajectory and repeats the process from the updated context.

\subsection{Dual-Branch Experience Memory}\label{sec:method-experience}

The weighting in Proposition~\ref{prop:weighting} depends on both the proposed actions and the accuracy of their predicted costs.
It cannot assign weight to a useful action absent from the candidate set, and inaccurate forecasts can favor unsuitable candidates.
Moreover, optimal candidate weights do not guarantee an optimal final action, since action selection is performed by a separate controller.
Our dual-branch memory aligns with these limitations: the planner branch targets candidate coverage and final action selection, while the world model branch adapts forecasts to environment-specific dynamics.

\begin{proposition}[Coverage, cost error, and selection]\label{prop:cost-error}
Fix a nonempty candidate set, its uniform reference $\mathbf u_t$, and $\lambda>0$.
Let $\bar C_t(a)$ denote a reference cost for any action $a$, with $\bar C_t^{(k)}=\bar C_t(a_t^{(k)})$, and let $\mathbf C_t$ be the predictive costs, with
\begin{equation}
\max_k|C_t^{(k)}-\bar C_t^{(k)}|\leq\epsilon.
\end{equation}
Define $\bar{\mathcal F}_t$ by replacing $\mathbf C_t$ with $\bar{\mathbf C}_t$ in Eq.~\eqref{eq:candidate-objective}, and let $\bar{\mathbf w}_t$ minimize $\bar{\mathcal F}_t$.
The predictive weights $\mathbf w_t^\star$ satisfy
\begin{align}
\bar{\mathcal F}_t(\mathbf w_t^\star)
-\bar{\mathcal F}_t(\bar{\mathbf w}_t)&\leq 2\epsilon,\label{eq:objective-gap}\\
D_{\mathrm{KL}}(\mathbf w_t^\star\Vert\bar{\mathbf w}_t)
&\leq \frac{2\epsilon}{\lambda}.\label{eq:kl-gap}
\end{align}
Moreover, for any comparator action $a^\circ$, possibly outside the candidate set, define the coverage gap and the selection gap
\begin{align}
\Delta_t^{\mathrm{cov}}&=\min_k\bar C_t^{(k)}-\bar C_t(a^\circ),\label{eq:coverage-gap}\\
\delta_t^{\mathrm{sel}}&=\mathbb E_{a_t\sim\pi_{\mathrm{final}}}\big[\bar C_t(a_t)\big]-\sum_{k=1}^{N_t}w_t^{(k)\star}\bar C_t^{(k)},\label{eq:selection-gap}
\end{align}
where the expectation is over the controller's output given its inputs in Eq.~\eqref{eq:final-controller}. Then the executed action satisfies
\begin{equation}
\begin{aligned}
&\mathbb E_{a_t\sim\pi_{\mathrm{final}}}\big[\bar C_t(a_t)\big]-\bar C_t(a^\circ)\\
&\quad\leq
\underbrace{\delta_t^{\mathrm{sel}}}_{\text{selection}}
+\underbrace{\Delta_t^{\mathrm{cov}}}_{\text{coverage}}
+\underbrace{2\epsilon}_{\text{cost error}}
+\underbrace{\lambda\log N_t}_{\text{temperature}}.
\end{aligned}
\label{eq:regret-decomp}
\end{equation}
\end{proposition}
The proof is in Appendix~\ref{app:weighting-proofs}.
Reference costs represent an ideal assessment of actual action consequences, which are not available to the agent.
Taking $a^\circ$ to be a reference-optimal action makes the left-hand side of Eq.~\eqref{eq:regret-decomp} a per-step regret and $\Delta_t^{\mathrm{cov}}\geq 0$. Eq.~\eqref{eq:regret-decomp} separates the limitations above: the coverage gap is positive only when a better action than every candidate exists, and shrinks as better candidates are proposed, while $\epsilon$ decreases with more accurate forecasts and cost assessments.
The selection gap compares the controller’s choice with the weighted candidates: it is negative when the controller corrects inaccurate weights or generates a better action, and positive when it disregards reliable guidance.
This decomposition guides our memory design.
Candidate-generation memories target $\Delta_t^{\mathrm{cov}}$, final-planner memories target $\delta_t^{\mathrm{sel}}$, and world model memories target $\epsilon$.
Both memory branches update textual memories while keeping the model parameters frozen.

\paragraph{Planner branch.}
After each real transition, an experience extraction module labels the step as SUCCESS, FAILURE, or IGNORE based on the observed outcome.
It then records experience only for SUCCESS and FAILURE, producing separate summaries for candidate generation and final action selection.
Each memory bank $\mathcal{W}_{d,y}$ retains its most recent $C$ records.
To retrieve memories, let $\mathcal{T}(\cdot)$ map text to its set of lowercase alphanumeric tokens, and let $\mathcal{Q}_t=\mathcal{T}(g)\cup\mathcal{T}(o_t)$ be the query tokens. We score each memory $m$ by token-set Jaccard similarity and keep the $R$ highest-scoring ones:
\begin{equation}
\begin{aligned}
s_t(m)&=\frac{|\mathcal{T}(m)\cap\mathcal{Q}_t|}{|\mathcal{T}(m)\cup\mathcal{Q}_t|},\\[2pt]
\mathcal{R}_{d,y,t}&=\operatorname*{top\text{-}\mathit{R}}_{m\in\mathcal{W}_{d,y}} s_t(m),
\end{aligned}
\end{equation}
where $d\in\{\mathrm{cand},\mathrm{final}\}$ indexes candidate-generation or final-planner experiences, and $y\in\{\mathrm{success},\mathrm{failure}\}$ indexes memory banks for experiences from successful and failed steps.
Accordingly, the planner experiences in Eqs.~\eqref{eq:proposal-sampling} and~\eqref{eq:final-controller} are $M_t^{\mathrm{P}}=(M_t^{\mathrm{cand}},M_t^{\mathrm{final}})$, where $M_t^{d}=\bigcup_{y}\mathcal{R}_{d,y,t}$; the proposal policy conditions on $M_t^{\mathrm{cand}}$ and the final controller on $M_t^{\mathrm{final}}$.

\begin{table*}[!ht]
\centering
\small
\setlength{\tabcolsep}{5pt}
\begin{tabular}{l!{\vrule width 0.4pt}lcccccc}
\toprule
 \multicolumn{1}{l!{\vrule width 0.4pt}}{\multirow{2}{*}[-0.4em]{Metric}\rule[-\dimexpr\dp\strutbox+\cmidrulewidth+1.2pt\relax]{0pt}{0pt}} & \multirow{2}{*}[-0.4em]{Method} & \multicolumn{2}{c}{\raisebox{-0.25\height}{\tikz[x=1.1em,y=1.1em]{\path[fill=blue!65!cyan] (0.5,1) .. controls (0.5,0.68) and (0.68,0.5) .. (1,0.5) .. controls (0.68,0.5) and (0.5,0.32) .. (0.5,0) .. controls (0.5,0.32) and (0.32,0.5) .. (0,0.5) .. controls (0.32,0.5) and (0.5,0.68) .. (0.5,1) -- cycle;}}\hspace{0.45em}\textit{Gemma-4-26B-A4B}} & \multicolumn{2}{c}{\raisebox{-0.25\height}{\includegraphics[height=1.1em]{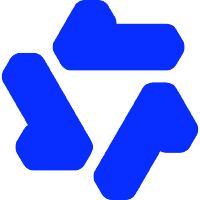}}\hspace{0.45em}\textit{Qwen3.5-9B}} & \multicolumn{2}{c}{\raisebox{-0.25\height}{\openailogo}\hspace{0.45em}\textit{GPT-5.4-mini}} \tabularnewline
\noalign{\global\aboverulesep=0pt \global\belowrulesep=0pt}
\cmidrule(lr){3-4} \cmidrule(lr){5-6} \cmidrule(l){7-8}
\noalign{\global\aboverulesep=0.4ex \global\belowrulesep=0.65ex}
 \multicolumn{1}{l!{\vrule width 0.4pt}}{\rule{0pt}{2.6ex}} & & ALFWorld & ScienceWorld & ALFWorld & ScienceWorld & ALFWorld & ScienceWorld \tabularnewline
\midrule
\multirow{4}{*}{SR@1} & ReAct & 20.1 & 31.1 & 14.9 & \textbf{34.4} & 38.8 & \underline{52.2} \tabularnewline
 & ITP-I & \underline{32.8} & \underline{36.7} & 6.7 & 27.8 & \underline{50.0} & 51.1 \tabularnewline
 & WorldEvolver & 16.4 & 32.2 & \underline{24.6} & 24.4 & 24.6 & 45.6 \tabularnewline
 & \cellcolor[HTML]{EAF3FB}\textbf{Ours} & \cellcolor[HTML]{EAF3FB}\textbf{65.7} & \cellcolor[HTML]{EAF3FB}\textbf{46.7} & \cellcolor[HTML]{EAF3FB}\textbf{28.4} & \cellcolor[HTML]{EAF3FB}\underline{32.2} & \cellcolor[HTML]{EAF3FB}\textbf{64.9} & \cellcolor[HTML]{EAF3FB}\textbf{56.7} \tabularnewline
\midrule
\multirow{4}{*}{SR@5} & ReAct & 24.6 & 36.7 & 24.6 & \textbf{52.2} & 48.5 & 62.2 \tabularnewline
 & ITP-I & \underline{54.5} & \underline{50.0} & 12.7 & 46.7 & \underline{81.3} & \textbf{70.0} \tabularnewline
 & WorldEvolver & 20.9 & 43.3 & \underline{37.3} & 42.2 & 30.6 & 56.7 \tabularnewline
 & \cellcolor[HTML]{EAF3FB}\textbf{Ours} & \cellcolor[HTML]{EAF3FB}\textbf{84.3} & \cellcolor[HTML]{EAF3FB}\textbf{60.0} & \cellcolor[HTML]{EAF3FB}\textbf{54.5} & \cellcolor[HTML]{EAF3FB}\underline{50.0} & \cellcolor[HTML]{EAF3FB}\textbf{92.5} & \cellcolor[HTML]{EAF3FB}\underline{66.7} \tabularnewline
\bottomrule
\end{tabular}
\caption{Overall success rates (\%) on ALFWorld and ScienceWorld across three backbones, reported as SR@1 and SR@5. Bold and underlined success rates indicate the best and second-best results, respectively. }
\label{tab:main}
\end{table*}

\paragraph{World model branch.}
For the world model branch, we adopt the episodic and semantic memory design of WorldEvolver~\citep{zhang2026worldevolver}, which we use to reduce the predictive cost error $\epsilon$ in Proposition~\ref{prop:cost-error}.
To support prediction for multiple candidate actions, we retrieve relevant transitions separately for each candidate and use them with shared environment rules to predict next observations.
Memory is updated only after real interaction, storing observed transitions and extracting rules from prediction errors.
The implementation details are provided in Appendix~\ref{app:world-memory}.

\section{Experiments}

\subsection{Experimental Setup}

\paragraph{Benchmarks.}
We evaluate our method on ALFWorld~\citep{shridhar2021alfworld} and ScienceWorld~\citep{wang2022scienceworld} provided by AgentBoard~\citep{ma2024agentboard}.
Our evaluation includes 134 ALFWorld tasks and 90 ScienceWorld tasks, with each trial limited to 30 environment steps.
ALFWorld is a text-based household benchmark of tasks such as picking up, cleaning, heating, and moving objects, while ScienceWorld is a text-based science benchmark where agents manipulate objects and carry out simple experiments.
Our primary metrics are SR@1, the fraction of tasks solved in a single trial, and SR@5, the fraction solved in at least one of five trials.

\paragraph{Baselines.}
We compare our method with multiple baselines.
ReAct~\citep{yao2023react} interleaves reasoning and action, using environment observations to guide subsequent decisions.
ITP-I, the training-free variant of Imagine-then-Plan~\citep{liu2026imagine}, uses adaptive lookahead to imagine future outcomes and guide action selection.
WorldEvolver~\citep{zhang2026worldevolver} represents an evolving language world model pipeline that uses episodic and semantic experience.
All baselines use the same task files and backbone.

\paragraph{Implementation details.}
At each step, the agent generates up to $K=8$ candidate actions and applies MPPI-style weighting with $\lambda=10$ to their predicted costs.
Candidate generation and final planning read from the corresponding planner memory bank with capacity $C=24$ and retrieval count $R=4$.
The policy and world model parameters remain frozen during deployment. Following WorldEvolver~\citep{zhang2026worldevolver}, we do not reset memory within each evaluation setting.
World model memory hyperparameters and prompt templates are in Appendices~\ref{app:world-memory} and~\ref{app:prompts}.
We also vary the backbone, evaluating Gemma-4-26B-A4B~\citep{gemmateam2026gemma4}, Qwen3.5-9B~\citep{qwen2026qwen35}, and GPT-5.4-mini~\citep{openai2026gpt54mini}, to test generalization across models.
In all tables, \textbf{bold} and \underline{underlined} numbers mark the best and second-best results per column, shaded rows are the default setting, and Avg. is the mean over ALFWorld and ScienceWorld.
\subsection{Results}

\paragraph{Results on ALFWorld and ScienceWorld.}
Table~\ref{tab:main} compares our method with ReAct, ITP-I, and WorldEvolver on ALFWorld and ScienceWorld across three backbones.
With Gemma-4-26B-A4B, Ours leads on both benchmarks, scoring 65.7/84.3 on ALFWorld and 46.7/60.0 on ScienceWorld, and improves over the strongest baseline, ITP-I, by 32.9/29.8 and 10.0/10.0 points, respectively.
With Qwen3.5-9B and GPT-5.4-mini, Ours achieves the best ALFWorld results and remains competitive on ScienceWorld: it ranks second behind ReAct on Qwen3.5-9B, and leads in SR@1 but ranks second in SR@5 on GPT-5.4-mini.
Across the three backbones, Ours achieves the strongest overall performance, ranking first in nine comparisons and second in the remaining three.
Per-task ALFWorld results are reported in Appendix~\ref{app:detailed-results}.

\begin{figure}[t]
\centering
\includegraphics[width=0.49\textwidth]{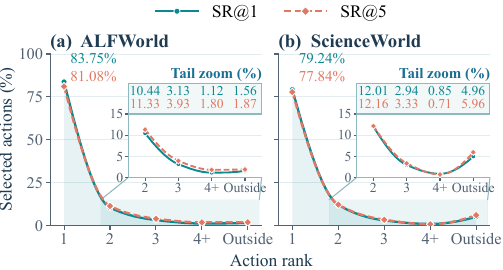}
\caption{Executed actions by candidate rank under MPPI-style weighting with Gemma-4-26B-A4B. Rank 1 includes ties for the highest weight; Outside denotes actions not in the proposal set.}
\label{fig:decision-distribution}
\end{figure}

\paragraph{Action selection analysis.}
Figure~\ref{fig:decision-distribution} shows that the controller selects a highest-weight candidate in 77.8--83.8\% of decisions across both benchmarks and trial settings.
Second-ranked candidates account for 10.4--12.2\%, while actions outside the proposal set account for 1.6--6.0\%.
The controller therefore usually follows the highest-weight candidate, while retaining the flexibility to choose other candidates or generate new actions.

\begin{figure*}[t]
\centering
\includegraphics[width=0.99\textwidth]{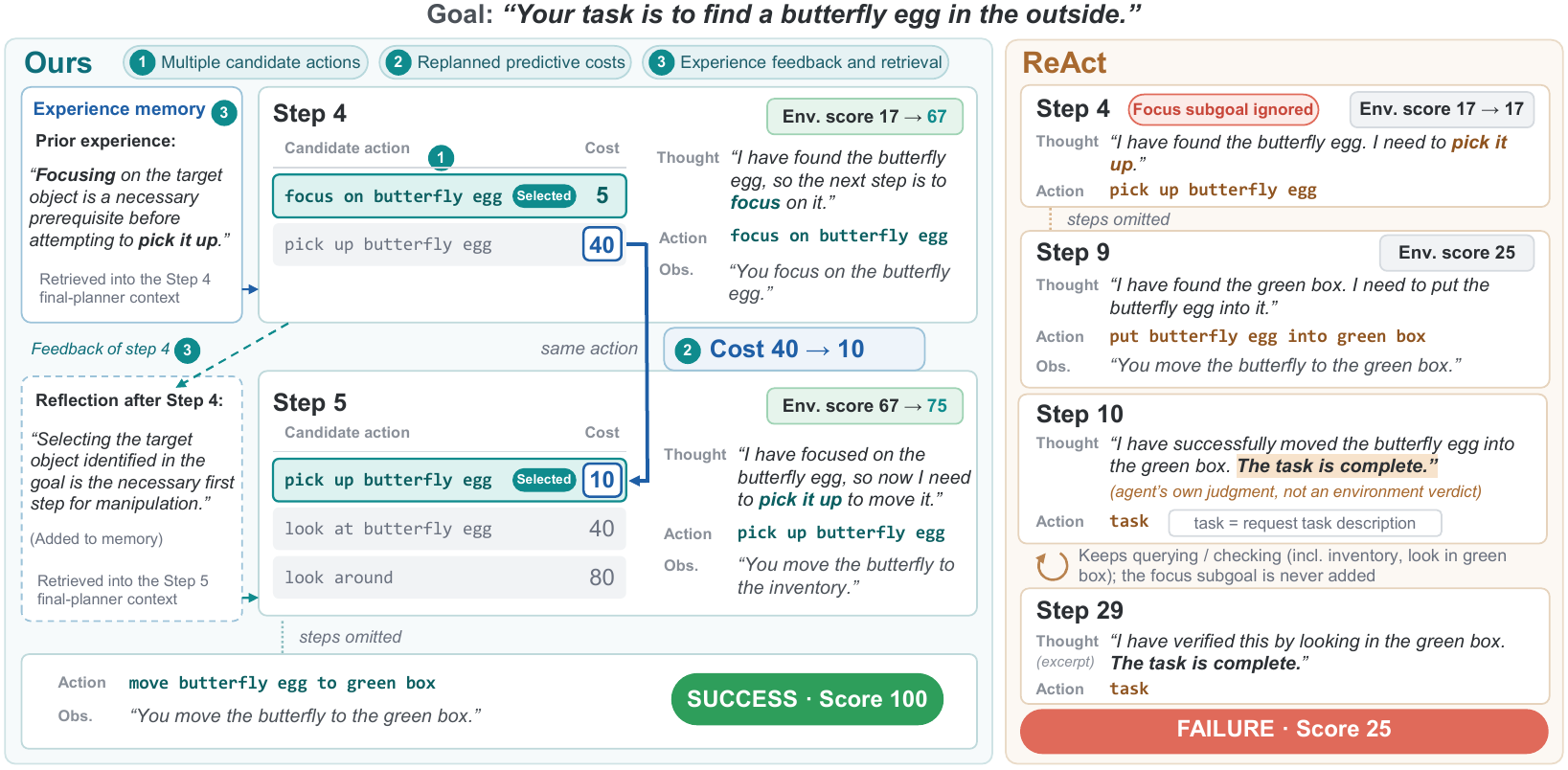}
\caption{An example of how retrieved experience and predictive evaluation guide action selection. Our agent completes the focus subgoal before pickup and succeeds, while ReAct skips this subgoal and fails.}
\label{fig:case}
\end{figure*}

\paragraph{Illustrative example.}
Figure~\ref{fig:case} shows how retrieved experience and predictive scoring guide a ScienceWorld trajectory.
At Step 4, retrieved experience identifies focusing on the target as a prerequisite for pickup, and the controller selects the lower-cost focus action over immediate pickup.
After this subgoal is completed, the predicted cost of pickup falls from 40 to 10, making it the preferred action at Step 5 and leading to environment-verified success with a score of 100.
In contrast, ReAct skips the focus subgoal and repeatedly declares completion despite ending with failure.

\subsection{Ablation Study}
\paragraph{Action candidate number $K$.}
Table~\ref{tab:ablation-control} evaluates candidate-set sizes $K\in\{1,2,4,8,16\}$.
The default $K=8$ gives the highest SR@1 on both benchmarks, exceeding $K=1$ by 5.3 percentage points on ALFWorld and 3.4 points on ScienceWorld.
Larger sets help SR@5 only on ScienceWorld, where $K=16$ raises it from 60.0 to 68.9 while SR@1 declines, possibly because of its larger and more varied action space.
Since larger candidate sets do not consistently help, we use $K=8$ in all settings as a moderate choice that balances candidate coverage and cost.

\begin{table}[t]
\centering
\small
\setlength{\tabcolsep}{5pt}
\begin{tabular}{ccccccc}
\toprule
\multirow{2}{*}{$K$} & \multicolumn{2}{c}{ALFWorld} & \multicolumn{2}{c}{ScienceWorld} & \multicolumn{2}{c}{Avg.} \tabularnewline
\cmidrule(lr){2-3} \cmidrule(lr){4-5} \cmidrule(l){6-7}
 & SR@1 & SR@5 & SR@1 & SR@5 & SR@1 & SR@5 \tabularnewline
\midrule
1 & 60.4 & 82.1 & \underline{43.3} & 58.9 & 51.9 & 70.5 \tabularnewline
2 & 59.7 & \textbf{85.1} & 38.9 & 56.7 & 49.3 & 70.9 \tabularnewline
4 & 61.9 & 82.8 & 40.0 & 58.9 & 51.0 & 70.9 \tabularnewline
\rowcolor[HTML]{EAF3FB}8 & \textbf{65.7} & \underline{84.3} & \textbf{46.7} & \underline{60.0} & \textbf{56.2} & \underline{72.2} \tabularnewline
16 & \underline{64.2} & \underline{84.3} & 42.2 & \textbf{68.9} & \underline{53.2} & \textbf{76.6} \tabularnewline
\bottomrule
\end{tabular}
\caption{Effect of the number of candidate actions $K$ imagined per step, with Gemma-4-26B-A4B (\%). The default setting is shaded.}
\label{tab:ablation-control}
\vspace{-0.2cm}
\end{table}

\paragraph{Weight temperature $\lambda$.}
Table~\ref{tab:ablation-lambda} evaluates $\lambda\in\{1,10,50\}$, where larger values produce flatter weights and reduce the influence of predictive cost differences.
Performance is similar at $\lambda=1$ and $\lambda=10$, whereas $\lambda=50$ degrades ALFWorld, lowering SR@1 from 65.7 to 53.7, and ScienceWorld is less sensitive.
We use $\lambda=10$ as the default, an intermediate temperature that keeps cost differences informative while softening the weights, and the results are not sensitive to this choice between $\lambda=1$ and $\lambda=10$.

\begin{table}[t]
\centering
\small
\setlength{\tabcolsep}{5pt}
\begin{tabular}{ccccccc}
\toprule
\multirow{2}{*}{$\lambda$} & \multicolumn{2}{c}{ALFWorld} & \multicolumn{2}{c}{ScienceWorld} & \multicolumn{2}{c}{Avg.} \tabularnewline
\cmidrule(lr){2-3} \cmidrule(lr){4-5} \cmidrule(l){6-7}
 & SR@1 & SR@5 & SR@1 & SR@5 & SR@1 & SR@5 \tabularnewline
\midrule
1 & \textbf{65.7} & \underline{81.3} & \textbf{48.9} & \underline{62.2} & \textbf{57.3} & 71.8 \tabularnewline
\rowcolor[HTML]{EAF3FB}10 & \textbf{65.7} & \textbf{84.3} & 46.7 & 60.0 & \underline{56.2} & \textbf{72.2} \tabularnewline
50 & 53.7 & 80.6 & \underline{47.8} & \textbf{63.3} & 50.8 & \underline{72.0} \tabularnewline
\bottomrule
\end{tabular}
\caption{Effect of weight temperature $\lambda$ with Gemma-4-26B-A4B (\%). Larger $\lambda$ gives more uniform MPPI-style weights. The default setting is shaded.}
\label{tab:ablation-lambda}
\end{table}

\paragraph{Experience ablations.}
Table~\ref{tab:ablation-memory} isolates the contributions of planner memory and world model memory.
Their effects differ across benchmarks: world model memory alone helps on ALFWorld, raising SR@1 from 56.7 to 68.7, but hurts on ScienceWorld, whereas planner memory alone helps on ScienceWorld but hurts on ALFWorld.
Enabling both branches achieves the highest average SR@1 and SR@5, though the gains are not uniform across benchmarks.
Figure 4 shows cumulative SR@1 as tasks are executed: the two curves coincide on the first tasks and then separate, with the gap persisting on both benchmarks, consistent with experience accumulating across trajectories.
Experience also takes effect within a single trajectory, as shown in Figure~\ref{fig:case}, where the experience summarized at one step is retrieved at the next.

\begin{table}[t]
\centering
\small
\setlength{\tabcolsep}{2pt}
\begin{tabular}{cccccccc}
\toprule
\multicolumn{2}{c}{Memory} & \multicolumn{2}{c}{ALFWorld} & \multicolumn{2}{c}{ScienceWorld} & \multicolumn{2}{c}{Avg.} \tabularnewline
\cmidrule(lr){1-2} \cmidrule(lr){3-4} \cmidrule(lr){5-6} \cmidrule(l){7-8}
Planner & WM & SR@1 & SR@5 & SR@1 & SR@5 & SR@1 & SR@5 \tabularnewline
\midrule
\textcolor{gray}{$\times$} & \textcolor{gray}{$\times$} & 56.7 & 80.6 & 38.9 & 57.8 & 47.8 & 69.2 \tabularnewline
$\checkmark$ & \textcolor{gray}{$\times$} & 49.3 & 73.9 & \underline{44.4} & \textbf{60.0} & 46.9 & 67.0 \tabularnewline
\textcolor{gray}{$\times$} & $\checkmark$ & \textbf{68.7} & \textbf{88.8} & 36.7 & 54.4 & \underline{52.7} & \underline{71.6} \tabularnewline
\rowcolor[HTML]{EAF3FB}$\checkmark$ & $\checkmark$ & \underline{65.7} & \underline{84.3} & \textbf{46.7} & \textbf{60.0} & \textbf{56.2} & \textbf{72.2} \tabularnewline
\bottomrule
\end{tabular}
\caption{Effect of planner and world model (WM) memory with Gemma-4-26B-A4B (\%). `$\checkmark$' indicates an enabled branch. The default setting is shaded.}
\label{tab:ablation-memory}
\end{table}

\begin{figure}
    \centering
    \includegraphics[width=0.99\linewidth]{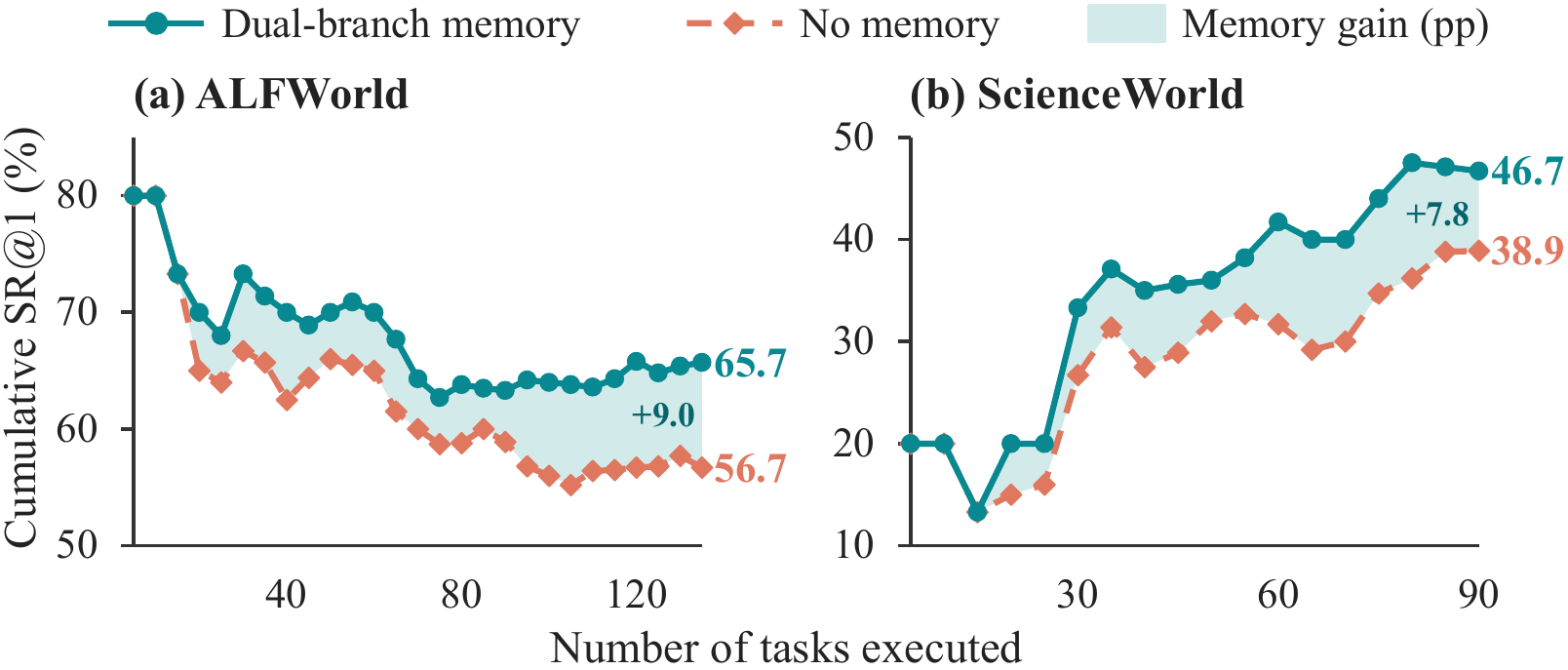}
    \caption{Cumulative SR@1 over executed tasks with and without dual-branch memory. Shaded areas show the gain in percentage points.}
    \label{fig:cumulative_sr1}
\end{figure}

\paragraph{Agent memory budget.}
Table~\ref{tab:ablation-agent-memory} examines the capacity $C$ and the number $R$ of records retrieved per bank of planner memory, so each planning call uses at most $2R$ records.
At $R=4$, increasing $C$ from 4 to 24 improves SR@1 on both benchmarks, while $C=36$ lowers it.
Performance is relatively stable across moderate budgets, namely $C\in\{12,24,36\}$ and $R\in\{1,4\}$, and $R=8$ gives mixed results.
We therefore use $(C,R)=(24,4)$ as a shared default with moderate storage and retrieval budgets.

\begin{table}[t]
\centering
\small
\setlength{\tabcolsep}{8.5pt}
\begin{tabular}{cccccc}
\toprule
\multirow{2}{*}{$C$} & \multirow{2}{*}{$R$} & \multicolumn{2}{c}{ALFWorld} & \multicolumn{2}{c}{ScienceWorld} \tabularnewline
\cmidrule(lr){3-4} \cmidrule(l){5-6}
 & & SR@1 & SR@5 & SR@1 & SR@5 \tabularnewline
\midrule
4 & 4 & 62.7 & \underline{85.8} & 36.7 & 60.0 \tabularnewline
12 & 4 & \underline{64.9} & \underline{85.8} & 44.4 & 58.9 \tabularnewline
\rowcolor[HTML]{EAF3FB}24 & 4 & \textbf{65.7} & 84.3 & \underline{46.7} & 60.0 \tabularnewline
36 & 4 & 63.4 & 84.3 & 42.2 & \textbf{64.4} \tabularnewline
\midrule
24 & 1 & 64.2 & \textbf{89.6} & \underline{46.7} & \underline{63.3} \tabularnewline
24 & 8 & 56.7 & 85.1 & \textbf{48.9} & 58.9 \tabularnewline
\bottomrule
\end{tabular}
\caption{Effect of planner memory capacity $C$ and retrieval count $R$ with Gemma-4-26B-A4B (\%). The upper block varies $C$ at $R=4$, the lower block varies $R$ at $C=24$. The default setting is shaded.}
\label{tab:ablation-agent-memory}
\end{table}

\subsection{Computational Cost}

Table~\ref{tab:cost} reports per-step latency measured on a task subset with Gemma-4-26B-A4B running on 2$\times$ NVIDIA A40 GPUs.
Our method takes 6.72 seconds per step, compared with 0.51--1.94 seconds for the baselines.
The additional overhead comes from predicting and scoring multiple candidate actions before execution.
Although slower, this remains practical for tasks without strict real-time constraints, where improved task success can justify the added cost.

\begin{table}[!ht]
\centering
\small
\setlength{\tabcolsep}{6pt}
\begin{tabular}{lccc}
\toprule
Method & Avg. sec. / step$\downarrow$ & SR@1$\uparrow$ & SR@5$\uparrow$ \tabularnewline
\midrule
ReAct & 0.51 & 20.1 & 24.6 \tabularnewline
ITP-I & 1.94 & \underline{32.8} & \underline{54.5} \tabularnewline
WorldEvolver & 1.33 & 16.4 & 20.9 \tabularnewline
\rowcolor[HTML]{EAF3FB}\textbf{Ours} & 6.72 & \textbf{65.7} & \textbf{84.3} \tabularnewline
\bottomrule
\end{tabular}
\caption{Per-step latency and success rates using Gemma-4-26B-A4B on 2$\times$ NVIDIA A40 GPUs.}
\label{tab:cost}
\end{table}

\section{Conclusion}
We presented a self-improving, test-time MPPI-inspired framework for language agents performing long-horizon tasks.
The framework predicts the outcomes of several candidate actions with a world model before acting, and uses MPPI-style weights to summarize their costs for the final decision.
A dual-branch experience memory enables self-improvement: planner memory refines candidate generation and selection, and world model memory improves next-observation prediction.
Experiments on ALFWorld and ScienceWorld demonstrate the strongest overall performance across three backbones, with measured computational costs remaining practical for the evaluated tasks.

\section{Limitations}

First, world model accuracy is important because predicted outcomes determine candidate costs and can change the controller's action preference.
Our current dual-branch adaptation relies on retrieved textual experience while keeping model parameters frozen. More capable mechanisms for improving world model accuracy and adapting to changing dynamics remain an important direction for future work.
Second, evaluating multiple candidates requires additional world model and evaluator calls at test time.
This increases inference latency and cost as the candidate set grows, motivating more computationally efficient ways to obtain useful candidate comparisons.

\textbf{Ethical considerations. }Our experiments are conducted entirely in text-based simulated environments and involve no human subjects or personal data. If such an agent were deployed in real environments, inaccurate world model predictions could lead to unsafe actions, so human oversight would be needed.

\section*{Acknowledgments}

This research is supported by the National Artificial Intelligence Research Resource Pilot Awards NAIRR250199, NAIRR260019, and NAIRR260077; AMD University Program’s AI \& HPC Cluster; NVIDIA Academic Grant Program; Lambda's Research Grant; and Google Cloud Research Credits Program. Computational resources are also provided by Delta and DeltaAI at the National Center for Supercomputing Applications~(NCSA) through ACCESS allocations CIS250012, CIS250816, and CIS251188.

\bibliography{references}

\clearpage
\appendix
\section{World Model Memory Implementation Details}\label{app:world-memory}

Our implementation adopts WorldEvolver's episodic and semantic memory design~\citep{zhang2026worldevolver} as auxiliary prediction context.
For each real transition, the episodic bank stores the current observation, executed action, and observation returned by the environment.
For a candidate action $a$, we tokenize it and each stored action $a'$ into token sets and score them by Jaccard similarity:
\[
s(a,a')=\frac{|\mathcal{T}(a)\cap\mathcal{T}(a')|}{|\mathcal{T}(a)\cup\mathcal{T}(a')|}.
\]
We retrieve up to the five highest-scoring transitions.
Each retrieved record is presented as a state--action--next-state example.
The predictor prompt also includes the task goal and current observation.

The semantic memory bank stores reusable rules about environment changes and persistence.
For each prediction--observation mismatch, the language model extracts at most one rule.
Up to eight recent rules guide subsequent predictions.

For both benchmarks, the episodic memory bank stores up to 100 transitions and retrieves up to 5 per candidate, while the semantic memory bank stores up to 128 rules and supplies up to 8 per prediction.
Memory is retained across trajectories within each environment, allowing later interactions to reuse accumulated experience.

\section{Theoretical Analysis and Proofs}\label{app:weighting-proofs}
We prove Propositions~\ref{prop:weighting} and~\ref{prop:cost-error}.

\begin{proof}[Proof of Proposition~\ref{prop:weighting}]
Suppress the time index and define
\[
\begin{aligned}
Z&=\sum_{k=1}^{N}u^{(k)}\exp(-C^{(k)}/\lambda),\\
w^{(k)\star}&=\frac{u^{(k)}\exp(-C^{(k)}/\lambda)}{Z}.
\end{aligned}
\]
Since $u^{(k)}=1/N$, these weights coincide with Eq.~\eqref{eq:candidate-weights}. Using
\[
\log\frac{w^{(k)\star}}{u^{(k)}}
=-\frac{C^{(k)}}{\lambda}-\log Z,
\]
we obtain
\[
\mathcal F(\mathbf w)
=\lambda D_{\mathrm{KL}}(\mathbf w\Vert\mathbf w^\star)
-\lambda\log Z.
\]
The KL divergence is nonnegative and vanishes only at $\mathbf w=\mathbf w^\star$, establishing the unique minimizer. Since $\mathcal F(\mathbf w^\star)=-\lambda\log Z$, subtracting this minimum gives the stated identity. We use the convention $0\log 0=0$.
\end{proof}

\begin{proof}[Proof of Proposition~\ref{prop:cost-error}]
For any distribution $\mathbf w$ on the fixed candidate set, the cost-error assumption implies
\[
|\mathcal F_t(\mathbf w)-\bar{\mathcal F}_t(\mathbf w)|
\leq\sum_k w^{(k)}\epsilon=\epsilon.
\]
Optimality of $\mathbf w_t^\star$ for $\mathcal F_t$ then gives
\[
\begin{aligned}
\bar{\mathcal F}_t(\mathbf w_t^\star)
&\leq\mathcal F_t(\mathbf w_t^\star)+\epsilon\\
&\leq\mathcal F_t(\bar{\mathbf w}_t)+\epsilon\\
&\leq\bar{\mathcal F}_t(\bar{\mathbf w}_t)+2\epsilon.
\end{aligned}
\]
This proves Eq.~\eqref{eq:objective-gap}. Applying Proposition~\ref{prop:weighting} to the reference costs gives
\[
\begin{aligned}
&\bar{\mathcal F}_t(\mathbf w_t^\star)
-\bar{\mathcal F}_t(\bar{\mathbf w}_t)\\
&\qquad=\lambda D_{\mathrm{KL}}(\mathbf w_t^\star\Vert\bar{\mathbf w}_t)
\leq 2\epsilon.
\end{aligned}
\]
Dividing by $\lambda>0$ establishes Eq.~\eqref{eq:kl-gap}.

For Eq.~\eqref{eq:regret-decomp}, nonnegativity of the KL term gives
\[
\sum_k w_t^{(k)\star}\bar C_t^{(k)}
\leq\bar{\mathcal F}_t(\mathbf w_t^\star).
\]
The proof of Proposition~\ref{prop:weighting} also gives
$\bar{\mathcal F}_t(\bar{\mathbf w}_t)=-\lambda\log\bar Z_t$, where
\[
\begin{aligned}
\bar Z_t
&=\frac{1}{N_t}\sum_k\exp(-\bar C_t^{(k)}/\lambda)\\
&\geq\frac{1}{N_t}\exp\!\left(-\frac{\min_k\bar C_t^{(k)}}{\lambda}\right).
\end{aligned}
\]
Therefore,
\[
\bar{\mathcal F}_t(\bar{\mathbf w}_t)
\leq\min_k\bar C_t^{(k)}+\lambda\log N_t.
\]
Combining these inequalities with Eq.~\eqref{eq:objective-gap} yields
\[
\sum_k w_t^{(k)\star}\bar C_t^{(k)}
\leq\min_k\bar C_t^{(k)}+2\epsilon+\lambda\log N_t.
\]
Adding $\delta_t^{\mathrm{sel}}$ to both sides turns the left-hand side into $\mathbb E_{a_t\sim\pi_{\mathrm{final}}}[\bar C_t(a_t)]$, where the expectation is over $a_t\sim\pi_{\mathrm{final}}(\cdot\mid x_t,\mathcal A_t,\mathcal E_t,M_t^{\mathrm P})$ in Eq.~\eqref{eq:final-controller}. Subtracting $\bar C_t(a^\circ)$ then yields Eq.~\eqref{eq:regret-decomp} by the definition of $\Delta_t^{\mathrm{cov}}$.
The selection gap captures the final controller's contribution to per-step regret, completing the decomposition into candidate coverage, predictive cost error, and action selection.
\end{proof}

\onecolumn

\section{Detailed Results}\label{app:detailed-results}

Table~\ref{tab:alfworld-tasks} reports SR@1 / SR@5 for each ALFWorld task type (PICK, CLEAN, HEAT, COOL, LOOK, and PICK2). Although the best method varies across individual task types, overall success rate is the primary measure of performance, and Ours achieves the best overall ALFWorld results across all three backbones.

\begin{table*}[!ht]
\centering
\small
\setlength{\tabcolsep}{3pt}
\begin{tabularx}{0.99\textwidth}{l*{7}{>{\centering\arraybackslash}X}}
\toprule
Method & PICK & CLEAN & HEAT & COOL & LOOK & PICK2 & Overall \tabularnewline
\midrule
\multicolumn{8}{c}{\raisebox{-0.25\height}{\tikz[x=1.1em,y=1.1em]{\path[fill=blue!65!cyan] (0.5,1) .. controls (0.5,0.68) and (0.68,0.5) .. (1,0.5) .. controls (0.68,0.5) and (0.5,0.32) .. (0.5,0) .. controls (0.5,0.32) and (0.32,0.5) .. (0,0.5) .. controls (0.32,0.5) and (0.5,0.68) .. (0.5,1) -- cycle;}}\hspace{0.45em}\textit{Backbone: Gemma-4-26B-A4B}} \tabularnewline
\midrule
ReAct & 95.8 / 100.0 & 0.0 / 0.0 & 0.0 / 0.0 & 0.0 / 0.0 & 0.0 / 0.0 & 23.5 / 52.9 & 20.1 / 24.6 \tabularnewline
ITP-I & 95.8 / 100.0 & 29.0 / 58.1 & 0.0 / 4.3 & 9.5 / 38.1 & 38.9 / 77.8 & 17.6 / 47.1 & \underline{32.8} / \underline{54.5} \tabularnewline
WorldEvolver & 87.5 / 91.7 & 0.0 / 0.0 & 0.0 / 0.0 & 0.0 / 0.0 & 0.0 / 0.0 & 5.9 / 35.3 & 16.4 / 20.9 \tabularnewline
\rowcolor[HTML]{EAF3FB}\textbf{Ours} & 91.7 / 100.0 & 74.2 / 90.3 & 30.4 / 43.5 & 95.2 / 95.2 & 50.0 / 83.3 & 41.2 / 94.1 & \textbf{65.7} / \textbf{84.3} \tabularnewline
\midrule
\multicolumn{8}{c}{\raisebox{-0.25\height}{\includegraphics[height=1.1em]{figures/qwen-logo.png}}\hspace{0.45em}\textit{Backbone: Qwen3.5-9B}} \tabularnewline
\midrule
ReAct & 20.8 / 33.3 & 6.5 / 19.4 & 13.0 / 17.4 & 23.8 / 28.6 & 22.2 / 22.2 & 17.6 / 29.4 & 14.9 / 24.6 \tabularnewline
ITP-I & 4.2 / 20.8 & 6.5 / 9.7 & 0.0 / 0.0 & 4.8 / 4.8 & 27.8 / 44.4 & 0.0 / 0.0 & 6.7 / 12.7 \tabularnewline
WorldEvolver & 58.3 / 70.8 & 32.3 / 41.9 & 8.7 / 17.4 & 0.0 / 0.0 & 33.3 / 72.2 & 5.9 / 17.6 & \underline{24.6} / \underline{37.3} \tabularnewline
\rowcolor[HTML]{EAF3FB}\textbf{Ours} & 54.2 / 91.7 & 22.6 / 35.5 & 0.0 / 17.4 & 23.8 / 57.1 & 33.3 / 50.0 & 41.2 / 88.2 & \textbf{28.4} / \textbf{54.5} \tabularnewline
\midrule
\multicolumn{8}{c}{\raisebox{-0.25\height}{\openailogo}\hspace{0.45em}\textit{Backbone: GPT-5.4-mini}} \tabularnewline
\midrule
ReAct & 95.8 / 100.0 & 0.0 / 0.0 & 69.6 / 82.6 & 4.8 / 28.6 & 0.0 / 0.0 & 70.6 / 94.1 & 38.8 / 48.5 \tabularnewline
ITP-I & 79.2 / 100.0 & 67.7 / 93.5 & 8.7 / 39.1 & 52.4 / 85.7 & 44.4 / 100.0 & 35.3 / 64.7 & \underline{50.0} / \underline{81.3} \tabularnewline
WorldEvolver & 87.5 / 100.0 & 0.0 / 0.0 & 13.0 / 13.0 & 0.0 / 0.0 & 0.0 / 0.0 & 52.9 / 82.4 & 24.6 / 30.6 \tabularnewline
\rowcolor[HTML]{EAF3FB}\textbf{Ours} & 83.3 / 95.8 & 48.4 / 90.3 & 65.2 / 82.6 & 76.2 / 95.2 & 66.7 / 100.0 & 52.9 / 94.1 & \textbf{64.9} / \textbf{92.5} \tabularnewline
\bottomrule
\end{tabularx}
\caption{Success rates (\%) on ALFWorld by task type across three backbones. Each cell reports SR@1 / SR@5. Bold and underlined overall success rates indicate the best and second-best results, respectively.}
\label{tab:alfworld-tasks}
\end{table*}

\renewcommand{\paragraph}[1]{\par\noindent\textbf{#1}\par\smallskip}\section{Prompt Templates}\label{app:prompts}

We use WorldEvolver's ReactAgent prompt~\citep{zhang2026worldevolver} as a starting point and substantially revise it for our framework. The resulting module-specific prompts are presented below.

\subsection{Candidate Generation}\label{app:prompt-candidate}
Both calls receive the shared context below. Call 1 appends a reasoning instruction; Call 2 appends the returned reasoning and candidate-generation instructions.

\paragraph{Shared context.}
\begin{lstlisting}[style=prompt]
Generate 1 to 8 distinct one-step action candidates.

Environment instruction:
{environment_instruction}

There is an interaction example (style and action/observation pattern only; do not copy its final Thought/Action output contract):
{fixed_demonstration}

This stage proposes candidates for a
downstream planner. It does not execute an action and must use the separate
candidate format at the end of this prompt.

Goal:
{goal}

Trajectory:
{trajectory}

Current observation:
{current_observation}


Learned planner experience (fallible; success and failure experiences):
- {retrieved_candidate_experience_1}
- {retrieved_candidate_experience_2}
- ...
Follow the public protocol exactly. Generate distinct one-step actions
that may reveal useful information or advance the task.

Prefer a grounded action that directly advances the goal.
Use an information action only when it can resolve relevant uncertainty.
Preserve the operation requested by the goal; do not replace it with an unrelated
information action. If multiple visible referents are plausible, include distinct
grounded alternatives.

Copy every non-empty argument verbatim from the current observation or an environment Observation in the trajectory. Never use names found only in the Goal, demonstrations, or earlier model-generated actions.
Do not repeat the latest unchanged failed action or add redundant candidates.
Candidates are suggestions, not guarantees.

If no grounded argumented action is available, output one useful zero-argument protocol action.



Keep candidates distinct and do not repeat the latest unchanged failure.
Return only:
CANDIDATE_1=<action>
CANDIDATE_2=<action>
...


\end{lstlisting}

\paragraph{Call 1: Reasoning.}
\noindent\textit{Input: Shared context + the following reasoning instruction.}
\begin{lstlisting}[style=prompt]
Analyze the visible evidence, protocol, goal, and learned planner experience briefly.
This is a private advisory analysis for a separate emitter: do not reproduce
instructions or demonstrations, and do not emit candidate markers or actions.
\end{lstlisting}

\paragraph{Call 2: Candidate Action Generation.}
\noindent\textit{Input: Shared context + the following.}
\begin{lstlisting}[style=prompt]
Advisory analysis from the candidate reasoner:
{candidate_reasoning}

Before emitting, correct any analysis that replaced an executable goal
operation with information gathering. If the goal operation's target is
uncertain, use the visible plausible referents as separate candidates; do not
emit only information-gathering actions when the requested operation itself
can test the visible referents.
Now emit the candidate proposals. Return up to the requested maximum number of lines in exactly
this form, with no Thought, Action, explanation, markdown, or text after the
last line:
CANDIDATE_1=<one atomic action>
CANDIDATE_2=<one atomic action>
...
\end{lstlisting}

\subsection{World Model Prediction}\label{app:prompt-world-model}

\begin{lstlisting}[style=prompt]
## Frame Axioms and Persistence Rules
- {semantic_rule_1}
- {semantic_rule_2}
{additional_semantic_rules}

## Retrieved similar past transitions

Current state: {retrieved_state_1}
Action: {retrieved_action_1}
Next state: {retrieved_next_state_1}

{additional_retrieved_transitions}

You are a world model for the {environment_name} environment. Given the agent's current observation, its proposed action, and the task goal, predict what the agent will observe next.

If a "## Frame Axioms and Persistence Rules" section appears above, treat each rule as a constraint: do not predict any change to a property the rules say does not change for the action being taken. Use the rules to filter out spurious changes.

If a "## Retrieved similar past transitions" section is provided, use those transitions as analogies for what can change after this action. Let the retrieved transitions guide what changes, but keep the prediction consistent with the current observation and task goal.

Output format: a single paragraph starting with "Prediction:" describing the next observation in the style of the environment's own text output.

Task: {goal}

State:
{world_model_observation}

Action: {candidate_action}

Prediction:
\end{lstlisting}

\subsection{Cost Evaluation}\label{app:prompt-evaluator}

\begin{lstlisting}[style=prompt]
Score all candidate actions jointly as decision evidence for an embodied-agent planner.

Goal:
{goal}

Trajectory:
{trajectory}

Current observation:
{current_observation}

Candidates and their uncertain WM predictions:
Candidate 1: {candidate_action_1}
WM predicted next observation: {predicted_next_observation_1}
Candidate 2: {candidate_action_2}
WM predicted next observation: {predicted_next_observation_2}
{remaining_candidate_prediction_blocks}

Relevant planner memory (fallible):
{retrieved_final_planner_memory_json}

Assign each candidate a cost from 0 to 100. A lower cost means more likely to make useful progress or reveal useful information for the goal; 100 means very poor. Consider the current observation and trajectory, expected progress, information value, and whether the predicted result is consistent with the visible evidence. WM predictions are uncertain suggestions, not facts.
Do not reject, filter, rewrite, or replace any action. Score every candidate, including weak ones, in one response.

Return only one line per candidate, in the original order:
COST_1=<number from 0 to 100>
COST_2=<number from 0 to 100>
...
\end{lstlisting}

\subsection{Final Action Selection}\label{app:prompt-final}

\begin{lstlisting}[style=prompt]
Choose the next action for an embodied agent.

Environment instruction:
{environment_instruction}

There is an interaction example (demonstrates reasoning/action/observation style; the final output contract below takes precedence):
{fixed_demonstration}

The environment instruction and example are references for this decision. Use the candidates and forecasts supplied below, then follow this planner's final output contract.

Goal:
{goal}

Trajectory:
{trajectory}

Current observation:
{current_observation}

Candidate proposals and uncertain WM forecasts:
Candidate 1: {candidate_action_1}
WM forecast: {predicted_next_observation_1}
MPPI-inspired evidence: weight={weight_1:.6f}
Candidate 2: {candidate_action_2}
WM forecast: {predicted_next_observation_2}
MPPI-inspired evidence: weight={weight_2:.6f}
{remaining_candidate_forecast_cost_weight_blocks}

Planner memory (model-generated and fallible; unresolved is a hypothesis, not hidden environment state):
{retrieved_final_planner_memory_json}

The current observation is authoritative; WM forecasts are uncertain evidence.
Prefer a grounded candidate that directly advances the goal. Use an information action only when it can resolve relevant uncertainty. If every candidate is insufficient, synthesize one new atomic action in the public protocol using visible names.
If no candidate is suitable and external actions are allowed, generate one grounded protocol action. Preserve the operation requested by the goal.

Before acting, ground every non-empty action argument by copying its complete name from Current observation or an environment Observation in the trajectory. Never use a name found only in the Goal, a demonstration, a WM forecast, or an earlier model-generated action. If no candidate is sufficient, any external action must follow the same grounding rule; otherwise choose a zero-argument information action.

Output exactly two lines:
Thought: <brief decision>
Action: <one action>
Reference examples do not override this request. Return exactly one line in the form Thought: <one brief sentence>. Action: <one action>. Do not repeat the trajectory or the examples. Do not emit an Observation or a second Thought/Action pair.
\end{lstlisting}

\subsection{Planner Reflection}\label{app:prompt-reflection}

\begin{lstlisting}[style=prompt]
Evaluate only the latest visible transition.
Goal: {goal}
Previous state:
{previous_observation}
Executed action:
{executed_action}
New state/observation:
{observed_next_state}

You may retain an environment-specific regularity when the feedback directly supports it, but do not import a rule that is not shown by the transition or treat an unsupported guess as fact.
Return SUCCESS only for substantive, visible progress toward the goal. An accepted action, a changed string, a successful inspection, or a newly exposed fact is not sufficient by itself. The new observation must either establish a previously unavailable goal-relevant fact/relation, make a visible state change that advances the goal, or visibly confirm a goal-relevant sub-result.
Return FAILURE only when the visible result shows rejection, impossibility, contradiction, or a failed attempt that yields a reusable mistake.
Return IGNORE when the action was accepted but was routine, repeated, neutral, or provided no reusable experience. If there is no substantive progress and no reusable experience from failure, use IGNORE rather than inventing experience.

When the outcome is FAILURE, extract two short experiences from this example. For the candidate-generation experience, consider whether the proposal was grounded in visible names and relations, whether its arguments had enough context, whether it repeated an unchanged hypothesis, and whether another observation or interaction would reveal more information. For the final action-selection experience, consider whether the action matched the latest evidence, whether its predicted consequence was only a weak hypothesis, whether the candidate frontier was insufficient, and what visible change should be preferred next.

When the outcome is SUCCESS, extract two short reusable experiences instead: one about what made the proposal produce the substantive progress, and one about why the selected action was a good choice given that progress. Generalize from the visible transition without assuming that the same object, action, or hidden state will recur. Preserve a convention only when this transition directly supports it, label it as context-dependent, and do not invent a hidden cause. Do not copy the observation or write generic advice such as only "verify."

Also write one short unresolved question or constraint for the current trajectory when the goal is not yet complete. It is a fallible reasoning aid, not a hidden-state assertion. If there is no unresolved issue, use an empty string. On SUCCESS, fill the two success-experience fields and leave the failure fields empty. On FAILURE, do the reverse. On IGNORE, leave all four experience fields empty. Always provide the unresolved field.

Be concise: each experience and unresolved value must be at most 20 words. Do not add fields or markdown fences.

Return JSON only:
{
  "outcome": "success|failure|ignore",
  "candidate_generation_experience": "candidate proposal experience, or empty string",
  "final_planner_experience": "final action-selection experience, or empty string",
  "unresolved": "one short current-trajectory question or constraint, or empty string"
}\end{lstlisting}

\end{document}